\documentclass[pdflatex,sn-apa]{sn-jnl}

\usepackage{graphicx}%
\usepackage{multirow}%
\usepackage{amsmath,amssymb,amsfonts}%
\usepackage{amsthm}%
\usepackage{mathrsfs}%
\usepackage[title]{appendix}%
\usepackage[dvipsnames]{xcolor}%
\usepackage{textcomp}%
\usepackage{manyfoot}%
\usepackage{booktabs}%
\usepackage{algorithm}%
\usepackage{algorithmicx}%
\usepackage{algpseudocode}%
\usepackage{listings}%

\theoremstyle{thmstyleone}%
\newtheorem{theorem}{Theorem}
\theoremstyle{thmstyletwo}%

\theoremstyle{thmstylethree}%

\usepackage{url}
\usepackage{hyperref}
\usepackage{microtype} 

\usepackage{mathtools} 

\usepackage{bm}

\usepackage{siunitx}
\usepackage{amsmath}  
\usepackage{amsthm}
\newtheorem{defin}{Definition}

\newtheorem{corollary}{Corollary}
\newtheorem{conjecture}{Conjecture}
\newtheorem{assumption}{Assumption}
\DeclareMathOperator{\doop}{\textit{do}}
\DeclareMathOperator{\pa}{pa}
\DeclareMathOperator{\Pa}{Pa}
\DeclareMathOperator{\sign}{\text{sign}}
\DeclareMathOperator{\tran}{^{\mkern-1.5mu\mathsf{T}}}
\newcommand{\inner}[1]{\langle #1 \rangle}

\usepackage{hyperref}
\hypersetup{colorlinks=true,linkcolor=blue!80,urlcolor=purple!80,citecolor=blue!80}
\usepackage[most]{tcolorbox}
\tcbset{textmarker/.style={%
        enhanced,
        parbox=false,boxrule=0mm,boxsep=0mm,arc=0mm,
        outer arc=0mm,left=6mm,right=3mm,top=7pt,bottom=7pt,
        toptitle=1mm,bottomtitle=1mm,oversize}}
\newtcolorbox{hintBox}{textmarker,
    borderline west={6pt}{0pt}{yellow},
    colback=yellow!10!white}
\newcommand{\hypothesis}[1]{\begin{tcolorbox}[width=\columnwidth, halign=left, colframe=black, colback=white, boxsep=0mm, arc=2mm] #1 \end{tcolorbox}}

\newcommand{\RaColor}{black}
\newcommand{\RbColor}{black}
\newcommand{\RcColor}{black}

\newcommand{\Ra}[1]{\textcolor{\RaColor}{#1}} 
\newcommand{\Rb}[1]{\textcolor{\RbColor}{#1}} 
\newcommand{\Rc}[1]{\textcolor{\RcColor}{#1}} 

\begin{document}



\title[Structural Causal Models Reveal Confounder Bias in Linear Program Modelling]{Structural Causal Models Reveal Confounder Bias in Linear Program Modelling}



\author*[1]{\fnm{Matej} \sur{Ze{\v{c}}evi{\'c}}}\email{matej.zecevic@tu-darmstadt.de}

\author[2,4,†]{\fnm{Devendra} \sur{Singh Dhami}}\email{d.s.dhami@tue.nl}

\author[1,2,3]{\fnm{Kristian} \sur{Kersting}}\email{kersting@cs.tu-darmstadt.de}

\affil[1]{\orgname{Technical University of Darmstadt}, \orgaddress{\city{Darmstadt}, \country{Germany}}}

\affil[2]{\orgname{hessian.AI}, \orgaddress{\city{Darmstadt},  \country{Germany}}}

\affil[3]{\orgname{DFKI}, \orgaddress{\city{Darmstadt},  \country{Germany}}}

\affil[4]{\orgname{Eindhoven University of Technology}, \orgaddress{\city{Eindhoven}, \country{Netherlands}}}


\abstract{The recent years have been marked by extended research on adversarial attacks, especially on deep neural networks. With this work we intend on posing and investigating the question of whether the phenomenon might be more general in nature, that is, adversarial-style attacks outside classical classification tasks. Specifically, we investigate optimization problems as they constitute a fundamental part of modern AI research. To this end, we consider the base class of optimizers namely Linear Programs (LPs). On our initial attempt of a naïve mapping between the formalism of adversarial examples and LPs, we quickly identify the key ingredients missing for making sense of a reasonable notion of adversarial examples for LPs. Intriguingly, the formalism of Pearl's notion to causality allows for the right description of adversarial like examples for LPs. Characteristically, we show the direct influence of the Structural Causal Model (SCM) onto the subsequent LP optimization, which ultimately exposes a notion of confounding in LPs (inherited by said SCM) that allows for adversarial-style attacks. We provide both the general proof formally alongside existential proofs of such intriguing LP-parameterizations based on SCM for three combinatorial problems, namely Linear Assignment, Shortest Path and a real world problem of energy systems. \vspace{.25cm}\\
$^\dagger$\textit{DSD contributed while being with 1 and 2 before joining 4.}}




\maketitle

\section{Introduction}\label{sec1}
Adversarial attacks have gained a lot of traction in recent years \citep{brendel2017decision,guo2019simple} as there has been a lot of focus on safety and robustness of machine learning (ML) systems. An interesting observation, though, is that deep neural networks or rather over-parameterized models are the center of attention for most of such adversarial attacks \citep{zugner2018adversarial,chen2018shapeshifter}. We argue that this view is incomplete or even too narrow in the sense that the phenomenon around adversarials is more general \emph{in nature} and actually depends on the problem setup. We conjecture that any differentiable perturbed optimizer (DPO) is prone to this new notion of attack similar to classical adversarials that we discuss in this paper. DPOs are a well studied, pragmatic approach to differentiability of general mathematical program (MP) solver by means of perturbation, consider \citep{papandreou2011perturb, berthet2020learning, gumbel1954statistical, bach2011learning} for reference. If our conjecture were to be true, then our new view on adversarial attacks would stand as a very general problem beyond learning to classify. While this might turn out to be more of a scientifically/mathematically valuable insight rather than practical implication, as we prove in this paper, there are examples that we can construct which are clearly of high relevance. As we will see, `classical' classification adversarial examples might still pose a higher significance in terms of research in deep learning as they pose a threat to trust and explainability, however, attacks on LPs certainly hold major significance as well if we consider that many real world applications of high relevance, such as energy systems or online navigation services, depend on them. There has been previous works where MPs such as LPs but also Mixed Integer Programs (MIP) \citep{wu2020stronger,tjeng2017evaluating} have been used to compute adversarial attacks but not where such optimization modules (LP, MIP) themselves have been confronted with the attacks. In fact, and also due to the recent interest in tightly integrating MPs and deep learning \citep{amos2017optnet, paulus2021comboptnet}, this extension of adversarial attacks beyond deep networks already significantly advances our understanding of adversarial attacks i.e., it is not just expressiveness that leads to uninterpretable solutions with counter-intuitive properties. These two key arguments serve as motivation to why studying adversarials in LPs (and more broadly MPs) is important beyond pure scientific inquiry.

Interestingly, it turns out that the new type of attack we formalize develops naturally from the Pearlian notion of Causality \citep{pearl2009causality} when starting from the formalism of classical adversarial attacks. Put differently, the mathematical theory of causality as given by Pearl provides the right formal tools to establish a reasonable interpretation of adversarial examples in LPs. Speaking of causality, the subject refers to a very general idea, in that understanding causal interactions is even central to human cognition and thereby of high value to science, engineering, business, and law \citep{penn2007causal}. In the last decade, causality has been thoroughly formalized in various instances \citep{pearl2009causality,peters2017elements,hernan2020causal}. At its core lies a Structural Causal Model (SCM) which is considered to be the model of reality responsible for data-generation. An SCM is a powerful model in that it is capable of many things. The SCM implies a graph structure over its modelled variables, \Rc{and furthermore when specified in its entirety it} can reason about (hidden) confounders, and of course handle both interventions and counterfactuals. The richness of the SCM has been crucial for its successful application for ML in marketing \citep{hair2021data}, healthcare \citep{bica2020time} and education \citep{hoiles2016bounded}. While we conjecture the applicability of our paper's results to the general class of DPO (which is an effective sub-class of MPs), the focus of this work will be to motivate, illustrate and finally prove formally that we can exploit an SCM's hidden confounders to construct a new type of attack based on the classical adversarial attacks in order to attack LPs---which is the very first, basic sub-class of MPs. We coin this new attack \textbf{Hidden Confounder Attacks}, since exploiting knowledge of hidden confounders is both a necessary and sufficient condition for the construction of these adversarial-style examples.

Overall, we make a number of key contributions:\ (1) We derive for the first time a novel, theoretical connection between causality's SCMs and LPs, by which we then (2) use the hidden confounders of the SCM to devise an adversarial-style attack---which we call Hidden Confounder Attack (HCA)---onto the LP showing that non-classification problems can be prone to adversarial-style attacks; (3) We study and discuss two classical LP families and one real world applied optimization problem to further motivate research on HCA and their potentially worrisome consequences if being ignored. For reproduction, we make our code repository publicly available.\footnote{ \url{https://github.com/zecevic-matej/Hidden-Confounder-Attacks}}

\section{Background and Related Work}\label{sec:bg}
In the following, we will briefly review the background on adversarial attacks as defined in their original setting of classification, then the formalism of LPs alongside two famous problem instances (linear assignment and shortest path, both of which we will use later on), and finally SCMs with their causal mechanism and hidden confounders. We use mathematical notation for (i) to \emph{precisely} specify and capture important ideas and (ii) to eventually prove our theoretical insights, however, the reader is invited to skip formal details as they are not central for grasping the new ideas proposed in our paper, yet, a consideration will provide technical understanding about assumptions, limitations and reach of what is being proposed.

\subsection{`Classical' Adversarial Attacks (Classification)}
We are in the setting of classification, specifically, image classification where the task of the model is to give the `right' label to a given image fed as model input. By using a simple optimization procedure, \citet{szegedy2013intriguing} were able to find what they called `adversarial examples', which they defined to be imperceptibly perturbed input images such that these new images were no longer classified correctly by the predictive neural model. Note how we specifically talk about \emph{neural} models here as in the regular deep learning context. \citet{goodfellow2014explaining} then proposed the Fast Gradient Sign Method (FGSM) that considers the gradient of the error of the classifier w.r.t\ to the input image. Mathematically, they investigated perturbations of the form
\begin{equation} \label{eq:adversarial}
    \pmb{\eta} := \epsilon \sign(\nabla_{\mathbf{x}} J(\mathbf{x}, y; \pmb{\theta})) \in \mathbb{R}^{w\times h\times c}
\end{equation} where $\mathbf{x}\in \mathbb{R}^{w\times h\times c}$ is the input image, $y\in \mathbb{N}$ a class label, $\pmb{\theta}$ are the neural function approximator parameter, $J\colon \mathbb{R}^{w\times h\times c} \times \mathbb{N} \rightarrow \mathbb{R}$ a scalar-valued objective function, $\sign{:} \mathbb{R} \rightarrow [-1,1]$ an element-wise sign function and $\epsilon \in \mathbb{R}$ a free-parameter. A perturbation $\pmb{\eta}$ would then account for mis-classification of the given predictive model $f(\mathbf{x};\pmb{\theta})$ i.e.,
\begin{equation} \label{eq:misclassification}
    f(\mathbf{x};\pmb{\theta}) = y \neq f(\mathbf{x} + \pmb{\eta};\pmb{\theta}).
\end{equation} The inequality in Eq.\ref{eq:misclassification} represents a possibly strongly significant divergence from the expected \emph{semantic} meaning (i.e., from a human inspector's perspective) of the class to be predicted. For example, imagine a photograph of a dog that is being classified by $f$ as a dog (that is $f$ predicts the label `dog'). However, the perturbed image $\mathbf{x}+\pmb{\eta}$ is not classified by $f$ as `dog' but rather as, say, `washing machine.' What came to a surprise for many is two-fold, (1) the new classification could be something drastically different e.g.\ not another animal like a cat but for instance a washing machine (2) from a human perspective the perturbed image would still lead to the same classification i.e., still a dog (put differently, the human inspector cannot tell a difference between the original and perturbed images). Naturally, said susceptibility (1-2) led to a significant increase in research interest regarding robustness (to adversarial examples) in neural function approximators evoking the narrative of ``attacks'' and subsequent ``defences'' on the inspected classification modules, as commonly found in alternate literature such as cyber-security \citep{handa2019machine}.

\subsection{Mathematical Programming / Optimization}
Selecting the best candidate from some given set with regard to some criterion (or objective) is a general description of MPs (or just `optimization'), which arguably lies at the core of machine learning and many applications in science and engineering since we are in search of models and solution that are somehow the `best.' Classification, e.g, can be considered as a special instance of mathematical programming, where the optimum is reached when the model is able to provide the correct label each time it is being queried. An important (if not, the most important and fundamental) optimization family are LPs that are concerned with the optimization of an objective function and constraints that are \emph{linear} in the respective optimization variables. LPs are being applied widely in the real world, e.g., energy systems \citep{schaber2012transmission}. 
More formally, the optimal solution of an LP is given by
\begin{align} \label{lp}
    \mathbf{x}^* &= \ \arg\max\nolimits_{\mathbf{x}} \text{LP}(\mathbf{x}; \mathbf{w},\mathbf{A},\mathbf{b}) \\ &= \ \arg\max\nolimits_{\mathbf{x}\in\mathcal{P}_{\mathbf{A},\mathbf{b}}} \inner{\mathbf{w},\mathbf{x}},
\end{align} where $\inner{\mathbf{a},\mathbf{b}}:= \mathbf{a}\tran\mathbf{b} = \sum_i a_ib_i \in\mathbb{R}$ is the inner product (dot product), $\mathbf{w}\in \mathbb{R}^n$ is called the weight/cost vector, $\mathbf{A}\in \mathbb{R}^{m\times n},\mathbf{b}\in \mathbb{R}^m$ are the constraint matrix and vector respectively, and finally $\mathcal{P}_{\mathbf{A},\mathbf{b}}$ is the solution polytope (or feasible region) i.e., the convex subset of state space $\mathbf{X}$ such that each $\mathbf{x}$ satisfies the constraints. Formally, $\mathcal{P}_{\mathbf{A},\mathbf{b}}:=\{\mathbf{x} | \mathbf{Ax} \leq \mathbf{b} \text{ and } \mathbf{x} \geq \mathbf{0}\}\subset \mathbb{R}^n$ (if clear from context, we abbreviate to $\mathcal{P}$). Table \ref{tab:lp} presents two classical problems that can be expressed as linear programs: the \emph{Linear Assignment} Problem (LA) and the \emph{Shortest Path} Problem (SP). Both problems formulate the optimization variable $\mathbf{x}\in \mathbb{R}^n$ with either $n=|A\times B|$ or $n=|E|$ to be a selector/indicator variable. That is, for LA we have the respective dimensions to mean matches between different worker and tasks, whereas for SP they denote the edges that are part of the fianl, selected path that should end up being the shortest path in the network. Although the original formulation of the LA and SP problems are actually \emph{integer} LP formulations, which are generally known to be NP-complete opposed to the less restrictive regular LPs since they require the solutions to be integers and not some arbitrary real number, both problems can still be solved in \emph{polynomial} time. However, extensions of regular SP like the Travelling Salesman or the Canadian Traveller problems are known to be NP- and PSPACE-complete respectively without any such benefits as LA/SP beg to offer.
\begin{table}[t]
\footnotesize
  \centering
  \scalebox{1.}{
    \begin{tabular}{ c p{1cm} | | p{1cm} c}
        $\begin{aligned}
        \forall i\in A: \sum_{j\in B} x_{ij} = 1 \\
        \forall j\in B: \sum_{i\in A} x_{ij} = 1
        \end{aligned}$ & & & $\displaystyle\sum_{(i,j)\in E} x_{ij} - \displaystyle\sum_{(i,j)\in E} x_{ji} = \begin{cases}
       1 &\quad\text{if} \ i=s \\
       -1 &\quad\text{if} \ i=t \\
       0 &\quad\text{else} \\ 
     \end{cases}$ \\
     
     $x_{ij} \in [0,1]$ & & & $x_{ij} \in [0,1]$ \\
     & & & \\
    \end{tabular}} \ \\ \ \\
 \caption{\textbf{Classical Problems formulable as LPs.} Linear Assignment (left; abbrev. LA) and Shortest Path (right; SP). In LA one matches ``workers'' to ``tasks'' according to their ``skills''. In SP one finds the ``quickest'', valid path (collection of edges) from some node $i$ to node $j$ within the given graph/network. The constraints on the left specify rules such as each worker can only have one task and any task can only have one worker, whereas the constraints on the right define a `valid' path, that is, if one enters a certain node, then one needs to also exit out of said node to continue with a valid path. \label{tab:lp}}
\end{table}

\subsection{Causality} 
The question of causality is a highly philosophical, timeless question and subject of study by the likes of Plato and his fellows, but only recently has the AI/ML community started investing more into causality as a means for the next generation of intelligent systems using new formalizations that capture certain key ideas rigorously. Following the Pearlian notion of Causality \citep{pearl2009causality}, an SCM is defined as a 4-tuple $\mathcal{M}:=\inner{\mathbf{U},\mathbf{V},\mathcal{F},P(\mathbf{U})}$ where the so-called structural equations
\begin{align}\label{eq:scm}
v_i = f_i(\pa_i,u_i) \in \mathcal{F}
\end{align} assign values (denoted by lowercase letters) to the respective endogenous/system variables $V_i\in\mathbf{V}$ based on the values of their parents $\Pa_i\subseteq \mathbf{V}\setminus V_i$ and the values of their respective exogenous/noise/nature variables $\mathbf{U}_i\subseteq \mathbf{U}$, and $P(\mathbf{U})$ denotes the probability function defined over $\mathbf{U}$. In other words, $f$ is the causal (possibly physical) mechanism that converts values of the parent variables to the values of the variable interest, or how Pearl says ``$V_i$ listens to $\pa_i$.'' We usually say that $\pa_i$ are the \emph{direct causes} of $V_i$. Note that, opposed to the Markovian SCM discussed in for instance \citep{peters2017elements}, the definition of $\mathcal{M}$ is semi-Markovian thus allowing for shared $U$ between the different $V_i$. Such a $U$ is also called \textbf{hidden confounder} since it is a common cause of at least two $V_i,V_j (i\neq j)$. Opposed to a hidden confounder, a ``common'' confounder would be a common cause from within $\mathbf{V}$ (that is, we would have a specific name for that given confounder, it would not be in $\mathbf{U}$). An important concept that the formalization later on will require is the concept of causal sufficiency. Following \citet{spirtes2010introduction}: ``The set of endogenous variables on which SCM $\mathcal{M}$ enacts is called \textit{causally sufficient} if there exist no hidden confounders.'' Put differently, if our modelled system has no unobserved confounders (or we simply assume it to be that way), then we can call our system causally sufficient. \Rc{While SCM provide a formalization (a language) for reasoning over (possibly hidden) confounders, the practical consideration of confounders is difficult and requires positing of (often times overly strong) assumptions. Both adjustment for confounders and the identification of confounding structures in the graph is a challenging task, which increases in difficulty when the confounder are unmeasured.} For completion's sake, we mention more interesting properties of any SCM. The SCM induces (i) a causal graph $G$, (ii) an observational/associational distribution denoted $p^{\mathcal{M}}$, and (ii) they can generate infinitely many interventional and counterfactual distributions using the $\doop$-operator\footnote{Loosely speaking, the $\doop$-operation ``overwrites'' structural equations.}.

\section{Step-by-Step Derivation of Adversarial Examples for Linear Programs}
This section covers our key contribution and develops it step-by-step from ground up. This main section is structured in the following manner: we first discuss MPs/DPOs and how our results are expected to transfer to those to motivate the overall research direction and justify the investigation of LPs as initial case. Then, secondly, we present how a naïve mapping of the classical adversarial attack framework fails for LPs in the sense that we could not claim it to be an adversarial example (or even something similar). Then, thirdly, we present how the tools from causality can provide additional semantics to formulate our new adversarial-style attack, which we refer to as Hidden Confounder Attacks (HCA). In the fourth subsection, we conclude with what we consider to be important mathematical insights for HCA.

\subsection{Overarching Hypothesis and the Importance of Differentiability}
In the past, different classes of MPs (LPs, MIPs) have been used defensively for verifying the robustness of neural learners to adversarial examples \citep{tjeng2017evaluating} and offensively for generating actual adversarial examples \citep{zhou2020generating}. Here, we are concerned with a fundamentally different research question: \emph{``How do adversarial attacks affect MPs\textbf{ themselves}?''}. That is, we turn the table and instead of considering MPs as a service to the system to be attacked, we consider the programs themselves to be the system under attack. Our overarching hypothesis for this and possible future work is the following:
\hypothesis{\textbf{Hypothesis:} \emph{Adversarial examples or attacks refer to a concept more general than that of classification in that it also affects MPs. Thereby, adversarial examples are a property of the problem specification and not per se a property of the expressiveness of deep models or of the classification task.}} To the best of our knowledge, we are the first researchers to ask and investigate this question thoroughly. Therefore, in order to establish an initial connection between adversarial attacks and MPs we will start off with the most basic class of MPs: Linear Programs. Since an adversarial attack typically depends on gradient/first-order information to determine where the perturbation (or attack) is most effective, we also require such first-order information from our LPs. One way to achieve this is to consider the class of ML models which inject some noise, which is distributed w.r.t.~some differentiable probability distribution, into the LP optimizer. These so-called perturbed models have differentiability properties because of that. Therefore, these perturbed models have also been considered for inference tasks within energy models \citep{papandreou2011perturb} and regularization in online settings \citep{abernethy2014online} as immediate consequence of said differentiability. Initial works in this research direction date back to the Gumbel-max \citep{gumbel1954statistical} and were recently generalized to \textit{Differentiable Perturbed Optimizers} (DPO) featuring end-to-end learnability \citep{berthet2020learning}. As stated in the initial section of this paper, we conjecture (that is believe to be true) that DPO are susceptible to the same (or similar) style of adversarial examples that we are developing in this paper. To formulate an LP optimizer, $\mathbf{x}^*(\mathbf{w}){=}\arg \max_{\mathbf{x}\in\mathcal{P}_{\mathbf{A},\mathbf{b}}} \inner{\mathbf{w},\mathbf{x}}$, as a DPO one requires only the existence of a random noise vector $\mathbf{z}\in \mathbb{R}^n$ with positive and differentiable density $p(\mathbf{z})$ such that for $\epsilon\in\mathbb{R}_{>0}$,
\begin{equation} \label{eq:perturbations}
    \mathbf{x}^*(\mathbf{\hat{w}})=\mathbb{E}_{p(\mathbf{z})} [\arg \max\nolimits_{\mathbf{x}\in \mathcal{P}_{\mathbf{A},\mathbf{b}}} \inner{\mathbf{w} + \epsilon \mathbf{z}, \mathbf{x}}],
\end{equation} where $\mathbf{x}\in\mathbb{R}^n$ is the optimization variable living in the solution polytope $\mathcal{P}_{\mathbf{A},\mathbf{b}}$ described by LP constraints $\mathbf{A},\mathbf{b}$, where $\mathbf{\hat{w}}:=\mathbf{w}+\epsilon\mathbf{z}$ is the peturbed LP cost parameterization, $\inner{\cdot, \cdot}\in \mathbb{R}\in\mathbb{R}^n$ the inner product and $\mathbb{E}_{p(\mathbf{x})}[f(\mathbf{X})]$ the expected value of random variable $\mathbf{X}$ under the predictive model $f$. Related work on differentiability of more general MPs like quadratic/cone programs \citep{agrawal2019differentiable} or linear optimization within predict-and-optimize settings \citep{mandi2020interior} generally rely on the Karush-Kuhn-Tucker (KKT) conditions. The general advantage of a perturbation method (as in Eq.\ref{eq:perturbations}) over the analytical approaches is its ``black-box'' nature i.e., we don't require to know what kind of MP we are dealing with, since we simply add stochasticity into the problem. In other words, we don't need to be experts on any specific MP (or any MP at all for that matter) to be able to reap the benefits of differentiability. This ``invariance'' to the underlying MP and the fact that differentiability is a necessary key concept behind adversarial attacks, leads to following (informally stated) conjecture.
\begin{conjecture}[\textbf{HCAs on MPs}{\normalfont, informal}]\label{con:mps}
Differentiability is a necessary condition for constructing Hidden Confounder Attacks on MPs (to be defined in Sec.\ref{subsec:math}).
\end{conjecture}

\subsection{Naïve Mapping or ``Trying to Make Sense of What Adversarial Examples Could Mean in LPs''}
Let's start our derivation of HCA by first providing a naïve mapping/perspective between/for the classical adversarial attack and the famous class of LPs known as Linear Assignment (LA), both of which have been previously introduced in Sec.\ref{sec:bg}. Mathematically, the following correspondence can be found,
\begin{equation}
\begin{alignedat}{4}\label{eq:mapping}
    \mathbf{x} + \pmb{\eta} &:= \mathbf{\hat{w}}, \quad &&f_{\theta} &&:= \mathbf{x}^*(\cdot), &&\\
    y &:= \mathbf{x}^*(\mathbf{w}), \quad &&J &&:= F(\mathbf{\hat{w}}, \mathbf{w}), &&
\end{alignedat}
\end{equation} where $\mathbf{x}$ is the feature vector (e.g.\ an image), $y$ the class label, $f_{\theta}$ the neural predictive model, and $J$ the cost function (e.g.\ mean-squared error)---all of these symbols follow the notation from \citep{goodfellow2014explaining}. From the LP perspective, we interpret $\mathbf{w}\in\mathbb{R}^{|A|\times|B|}$ as describing the suitability of worker $a\in A$ for job $b\in B$ and the optimal solution $\mathbf{x}^*(\mathbf{w})\in [0,1]^{|A|\times |B|}$ as indicators highlighting the matched pairs $(a_i,b_j)$. The only addition we have to make to achieve the naïve mapping to adversarials is some distance measure $F$ acting on the original $\mathbf{x}^*(\mathbf{w})$ and the expected perturbed solution $\mathbf{x}^*(\mathbf{\hat{w}})$. This is necessary since we need to allow for a ``class'' change between solutions that should occur (or are caused) through an adversarial attack i.e., in the case of LA, $F$ could be for instance the Structural Hamming Distance \citep{hamming1950error}. Also note, in slight abuse of notation, our program solver $\mathbf{x}^*(\cdot)$ denotes $\arg\max\nolimits_{\mathbf{x}} \text{LP}(\mathbf{x}; \mathbf{w},\mathbf{A},\mathbf{b})$.

Having completed the naïve mapping in LA specifically, we could then naïvely view each optimal matching ``code'' $\mathbf{x}^*(\mathbf{w})$ as a certain class (or label) and then the gradient $\nabla_{\mathbf{w}} F$ could be used for performing an `adversarial attack' such that the `class' changes (significantly) while the input remains approximately the same. As one can easily make out, the major problem being faced here is that there exists no ``semantic impact'' to be observed for the human inspector akin to a neural network wrongly classifying a dog (small animal) as a plane (big travel machine). In other words, there is change but said change is not significant (or surprising, and thus not interesting). To make this point more clear, we summarize our key insight about adversarial examples obtained by looking at our naïve mapping onto LPs:
\hypothesis{\textbf{Missing Semantic Component.} \textit{The human inspector's invariance to the adversarial example is characteristic of an adversarial attack (e.g.\ still dog-looking image being classified as depicting a washing machine), while a naïve mapping to LPs as in (\ref{eq:mapping}) leaves out said human component. In other words, there is no general human expectation towards different optimal solutions to an LP since humans measure LPs only on their objective and therefore different optimal solutions are not `different.'}}

\subsection{Concepts from Causality Provide the Missing Semantic Component}
In the previous section we concluded that there is a missing semantic component when naïvely mapping between adversarials and LPs. Yet, the optimal solutions when considered in terms of codes as in the LA example will actually significantly \emph{deviate} from each other. This deviation (or difference) seems to suggest that there exists some more `fundamental' difference in solution albeit not for the specific optimization objective at hand since the cost values will remain similar (or even the same). But as suggested by the missing semantic component, a human inspector will have no general expectation towards either of the LP solutions. To put it differently, ``they look different but that is that.'' Nonetheless, to complete the picture it is worth taking a step back and observing the LP from a `meta' level perspective. \emph{On this meta level, we can ask the question of how the LP cost vector $\mathbf{w}$ was provided in the first place.} Here causality and its SCM come into play. The SCM $\mathcal{M}$ defines a mechanistic data-generating process which will generate the observational probability distribution $p^{\mathcal{M}}$ that the human modeller usually observes an empirical fraction of, denoted as data $\mathbf{D}$. So, to loop back to the meta-question of how the LP parameterization came to be, we observe the following relation for some function $\phi$:
\begin{equation}\label{eq:parameters}
    \phi(\mathbf{D}) = \mathbf{w} \quad \text{where} \ \mathbf{D}\sim p^{\mathcal{M}}.
\end{equation} According to Eq.\ref{eq:parameters}, the human modeller that takes the observed data as a basis for determining the cost vector of the LP and then uses some function-mapping between the SCM and the LP denoted as $\phi$ to produce said cost vector. To give a concrete example of such a modelling, consider Fig.\ref{fig:parameters} in which the human modeller observes a data set $\mathbf{D}:=\{(h_i,p_i)\}_i^n\sim p(H,P)=p^{\mathcal{M}}$ but \textbf{no} other information. The human modeller \emph{does neither} observe the SCM induced distribution $p^{\mathcal{M}}$ \textit{nor} any more information about the partial SCM $\mathcal{M}$ which would include knowledge on the structural equations $f_H,f_P$ and the hidden confounder $U_C$ (i.e., $U_C$ is shared by both equations). Therefore, also no information on the true SCM $\mathcal{M}^*$ either, where $U_C$ would be part of the endogenous variables (i.e., there would also be $f_W$ and all $U_C$ being replaced by $W$ standing for Wealth). The only knowledge available is the data set $\mathbf{D}\sim p^{\mathcal{M}}$ which numerically describes the Health ($H$) and a general notion of Vaccine Priority ($P$) of certain individuals. Note that a simple linear regression on the data shown in Fig.\ref{fig:parameters} reveals the existence of a causal relation between $H$ and $P$ (but not the direction). The true causal direction is read as ``lower health values cause higher priority values'' and written $H\rightarrow P$. However, $P$ describes other factors as well (for instance the age of an individual) since $P$ also depends on $U_P$ (and not only on its cause $H$ and the hidden confounder $U_C$).
\begin{figure}[t!]
    \centering
    \includegraphics[width=1\columnwidth]{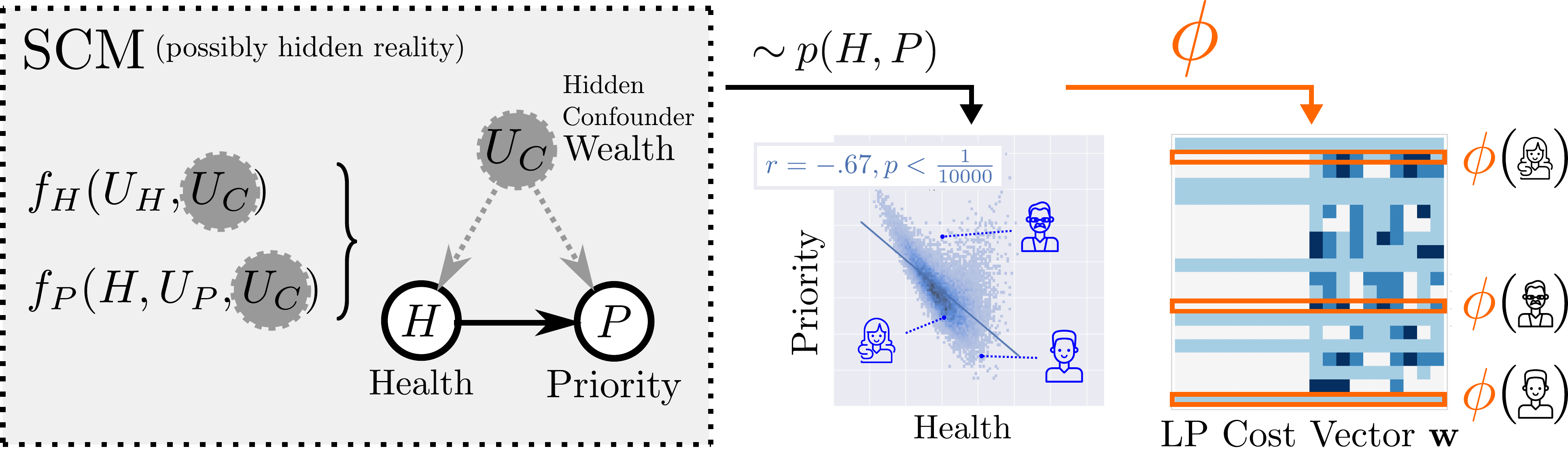}
    \caption{\Ra{\textbf{Intriguing LP-Parameterizations Based on SCM.} The observed observational distribution $p(H,P)$ was generated by some unobserved SCM. Even if we had \textit{some} SCM, it might not be the actual underlying, complete SCM i.e., there could still be hidden confounders in our estimate. The cost vector $\mathbf{w}$ of an LP can be viewed as a function $\phi$ applied to population individuals $(h,p)\sim p(H,P)$. \textit{Intriguingly, $\phi$ may very well be unaware of confounders in the true SCM.} (Best viewed in color.)}}
    \label{fig:parameters}
\end{figure} 

As in the previous section, let's consider an LA problem as our LP instance. We will make the example explicit. The setting is that of a pandemic and recently it has been announced that there is a new vaccine that can help in stopping the pandemic. For this, people need to get vaccinated and so it is now up to a human modeller scheduled for making a plan on assigning the empty vaccine spots. Unfortunately, the amount of spots is limited and so there must be some sort of rule to decide how to actually assign the available free spots. The human modeller is trying to find the \emph{optimal matching} of individuals (based on the data that covers relevant information/features about the individuals) to respective, available vaccination spots. The human modeller might choose to find the cost parameterization of the LP that will determine the optimal matching by following some sort of \emph{policy} (or rule) like \textit{``individuals of low health and high priority should be matched to vaccine spots, while others can wait}.'' More importantly, in this case, we would now argue that the human modeller \emph{implicitly} performed a mapping $\phi$ as in Eq.\ref{eq:parameters} based on the observed data, and that this $\phi$ essentially captures the policy description from before. Both interestingly and intriguingly, by construction, $\phi$ \emph{does not} consider the hidden confounder $U_C$---that might be viewed as something like wealth of an individual---since $U_C$ is not even defined within the data distribution accessible to the human modeller (since $\mathcal{M}\neq\mathcal{M}^*)$. This ignorance is the key to defining a meaningful adversarial-style attack on LPs since we can use it to explicate what a change in optimal solutions could mean. In informal terms, we are now ready to formulate our main idea of the paper:
\begin{theorem}[\textbf{HCAs on LPs}{\normalfont, informal}]\label{thm:informal}
Let $\phi_{\mathcal{M}}$ denote an LP parameterization based on SCM $\mathcal{M}$ while $\mathcal{M}^*$ denotes the `true' or optimal SCM for the phenomenon of interest. Any $\phi_{\mathcal{M}}$ that identifies individuals in $\mathcal{M}$ is prone to Hidden Confounder Attacks (to be defined in Sec.\ref{subsec:math}) unless $\mathcal{M}=\mathcal{M}^*$.
\end{theorem} Put loosely, choosing the `wrong' $\phi_{\mathcal{M}}$, one that does not represent the true, underlying SCM $\mathcal{M}^*$, allows for an adversarial-style attacks on LPs. In summary we state:
\hypothesis{\textbf{Bias in Vaccine Spot Assignments, Example from Fig.\ref{fig:leadex}:} \textit{The LP-parameterization $\phi_{\mathcal{M}}$ based on the SCM $\mathcal{M}$, with $\mathbf{w}=\phi(\mathbf{D})$ from Eq.\ref{eq:parameters} where $\mathbf{w}$ is the LP cost vector, follows the previously, informally described policy, so $\phi_{\mathcal{M}}$ models only $H,P$ since $U_C$ is not even defined in $\mathbf{D}$. Therefore, Thm.\ref{thm:informal} predicts the existence of a HCA for $\textnormal{LP}(\mathbf{x};\mathbf{w},\mathbf{A},\mathbf{b})$. Fig.\ref{fig:leadex} reveals such an example HCA: the attack is unnoticable in visual terms, just as for a classical adversarial example, and so is the difference in cost w.r.t.\ the optimal matchings---however, w.r.t.\ to the wealth value of each individual, the new matching shows a significant skew towards individuals (or data points) of higher wealth value.}} 

\Rc{\textbf{Rationale for Hidden Confounders within Mathematical Programming.} Considerably orthogonal at first glance as this bridge might seem, both Pearlian causality (and associated concepts such as hidden confounders) and Mathematical Programming are concerned with \emph{modelling assumptions} that lie outside the data. While the MP is arguably data-independent per definition of the paradigm, MPs such as LPs can in fact be modelled through data. To give an example, a warehouse that is concerned with profits will record its sales and eventually settle on optimizing the profits based on the recorded data (using both the data features or variables and insights on customer preferences), which in turn is formulated as an MP. Since data that we are given (or that is being recorded in the example of the warehouse) is assumed to be governed by some underlying dynamics, it is safe to assume that we can denote said underlying data-generating process with causality's center-piece formalism being that of SCM. In conclusion, we naturally find situations in which our approximation to the SCM that might be underlying our data (what is typically called our \emph{induction hypothesis}) is insufficient in that there are hidden confounders. However, since we used the SCM's data to model an MP previously, we now have an eventuality of hidden confounders within our MP instance. This rationale is the very basis of the present manuscript that proves the existence of hidden confounders within MPs that are modelled in accordance to an insufficient SCM (even in the cases where the MP generation is unaware of the assumptions placed on the SCM part of the modelling of the cost and constraint vectors).}

Next, to complete the discussion, we finally formalize the informal notions presented in this section. However, it is worth noting that technical parts of this paper can be safely skipped as understanding them is not central to understanding the overall idea (as just now presented informally). Naturally, reading the technical part allows for a precise understanding of the assumptions and key aspects to our definition of HCA and subsequent insights. 
\begin{figure*}[!t]
    \centering
    \includegraphics[width=1\textwidth]{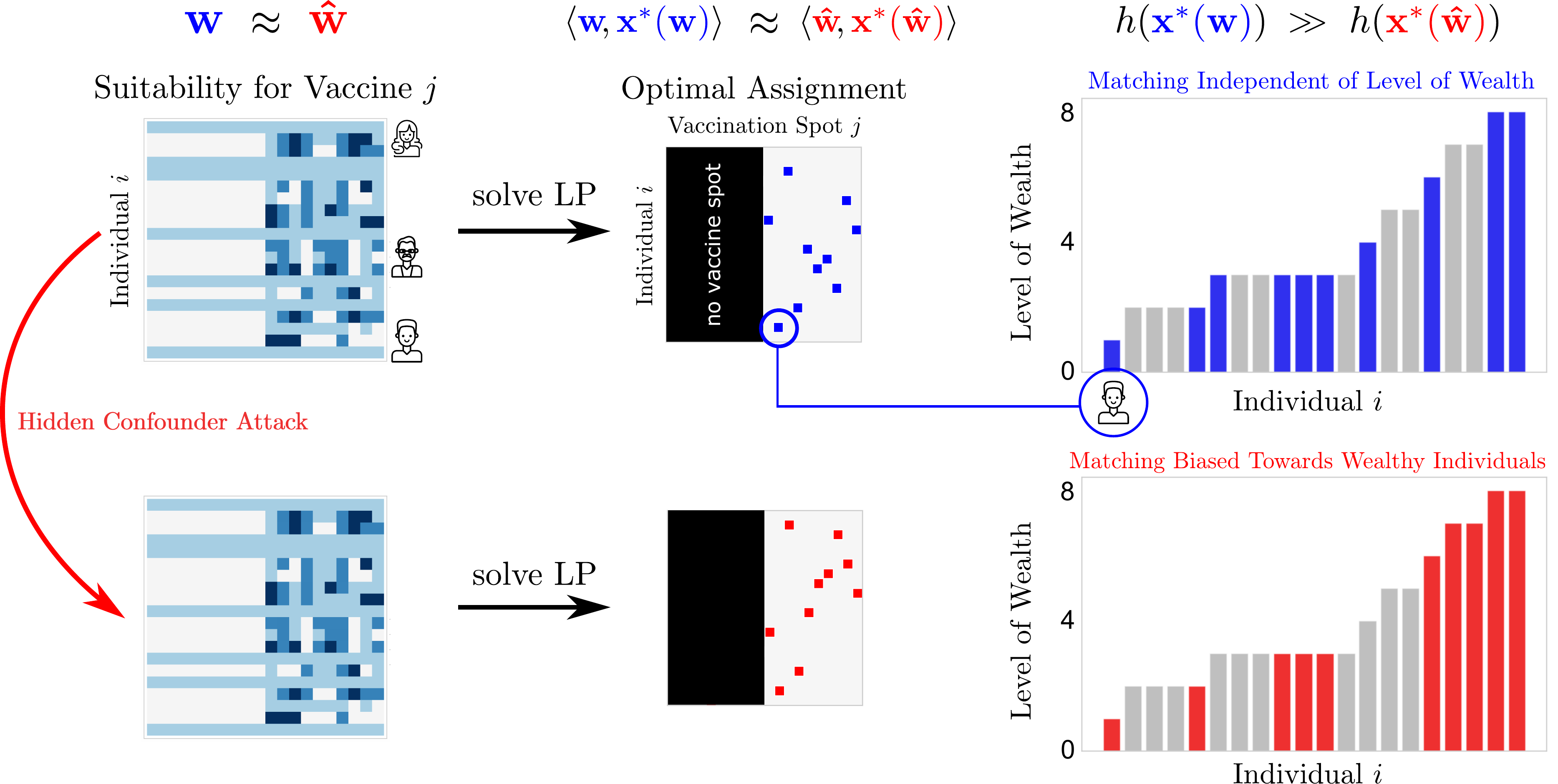}
    \caption{\Ra{\textbf{Lead Example: Hidden Confounder Attack on LP.} 
    A real world inspired matching problem under attack. The new matching shows significant bias for individuals with high wealth value. The adversarial LP cost vector is close to the original both value-wise (\textbf{left}) and cost-wise w.r.t.~their optimal solution (\textbf{middle}) which means in words that health-wise people in higher need of vaccination are still guaranteed a vaccine spot. However, w.r.t.~hidden confounder (Wealth) the adversarial solution drastically deviates (\textbf{right}) i.e., the distribution of vaccines is being skewed towards people of higher wealth, which can circumvent the originally intended policy. (Best viewed in Color.)
    }}
    \label{fig:leadex}
\end{figure*}

\subsection{Formalizing the Newly Proposed LP Attacks ``Hidden Confounder Attacks''} \label{subsec:math}
\Ra{\textbf{Notation.} We use standard notation from deep learning for various mathematical base concepts such as scalars ($a$), vectors ($\mathbf{a}$), matrices ($\mathbf{A}$), sets ($A$), parameterized functions ($f_{\theta}$). In the cases where a set is actually denoted with (what is considered typically to be) matrix notation, we consider said set to be of special meaning in that it is literature-specific notation where the different considered literatures are adversarial examples, linear programming and causality. In the following, we go over each. From literature specific to adversarial examples we use: perturbation ($\pmb{\eta}$). We also make use of linear programming specific notation: optimal solution ($\mathbf{x}^*$), cost vector or matrix ($\mathbf{w}$), inner product ($\langle\cdot,\cdot\rangle$), linear constraints ($\mathbf{A},\mathbf{b}$). Finally, since HCA are based on causality's conception of confounders, we use causality notation as well: SCM ($\mathcal{M}$), exo- and endogenous variables ($\mathbf{U},\mathbf{V}$). With the matching of the two paradigms as in Eqs.\ref{eq:mapping}, we can consider the notational extensions for HCA introduced in this work as a `generalization' of the previous notations deployed for discussing adversarial-style attacks. In other words, either side of the definitions (in the mapping of Eqs.\ref{eq:mapping}) can be used to denote HCA. The authors opt for notational convenience in the sense that notation which is useful for a given context should be deployed e.g.\ when discussing optimal solutions then $\mathbf{x}^*(\mathbf{w})\neq \mathbf{x}^*(\mathbf{w}')$ is easier to parse than $y \neq y'$ since the former notation explicates that optimal solutions under different cost vectors are being considered, whereas the latter might only be interpreted as a difference in scalars.} 
We now introduce extended fosrmalism to capture all the previously established ideas precisely, to define HCA and discuss its properties. We first start with Eq.\ref{eq:parameters}.
\begin{defin} \label{def:phi}
We call a function $\phi_{\mathcal{M}}$ LP-parameterization based on SCM $\mathcal{M}$ if for an observational data set $\mathbf{D}\sim p^{\mathcal{M}}$ we can define an $\textnormal{LP}(\mathbf{x}; \phi(\mathbf{D}),\mathbf{A},\mathbf{b})$.
\end{defin}
\noindent By default (unless clear from context), we might refer to such $\phi_{\mathcal{M}}$ simply as LP-parameterizations. Some LP-parameterizations fulfill a property that ``identifies individuals in $\mathcal{M}$'' which we define next.
\begin{defin} \label{def:integral}
Let $\phi_{\mathcal{M}}$ be an LP-parameterization based on SCM $\mathcal{M}$ with $\phi_{\mathcal{M}}(\mathbf{D})=\mathbf{w}=(w_1,...,w_k)$ and $\mathbf{D}=\{\mathbf{d}_i\}^n_{i}$. We call $\phi_{\mathcal{M}}$ integral if it satisfies:
\begin{equation*}
    \forall i\in \{1,...,n\}\ldotp\exists I\subset\{1,...,k\}\ldotp(\mathbf{d}_i = \phi_{\mathcal{M}}^{-1}((w_j)_{j\in I}).
\end{equation*}
In words, the parameterization decomposes on the the data point (or unit) level.
\end{defin}
\noindent These first two definitions give us the ability to talk about `special' types of LPs, namely those that underly some SCM and further some that even allow for talking about the different units $U_j$ of these SCM (exogenous variables). A simple yet important insight that immediately follows is:
\begin{corollary} \label{cor:integral}
The LP problems Linear Assignment (LA) and Shortest Path (SP) have integral LP-parameterizations.
\end{corollary}
\begin{proof}
For LA, simply map each data point indexed by $i$ to the cost vector slice indexed by indices in the set $I$ s.t.\ each $i$ refers to the same unique $a$ from the ``workers" set $A$ for all the different ``jobs" $b\in B$. I.e., one data point sampled from the SCM's observational distribution will correspond to one worker in the LP cost vector. For SP, there is a more direct one-to-one correspondence where each data point is mapped to a unique edge in the graph.
\end{proof}
\noindent This insight is important since LA and SP constitute arguably the two most important LPs in existence as accounted by their widespread use in applications in ML and beyond. Before we can define HCA, we need to state our two main assumptions that are necessary to HCA:
\begin{assumption} \label{ass:epsball}
For some fixed constraints $\mathbf{A},\mathbf{b}$ let $\mathbf{X}^*(\mathbf{w}):=\{\mathbf{x}\mid\mathbf{x} = \arg\max\nolimits_{\mathbf{x}} \textnormal{LP}(\mathbf{x}; \mathbf{w},\mathbf{A},\mathbf{b})\}$ denote the set of optimal LP solutions under $\mathbf{w}$. Further let $B^\mathbf{w}_{\epsilon}$ denote an $\epsilon$-Ball with $\epsilon>0$ around some LP cost $\mathbf{w}$. Then there exists a $\mathbf{\hat{w}}\in B^\mathbf{w}_{\epsilon}$ such that $|\mathbf{X}^*(\mathbf{\hat{w}})|>1$. In words, we can find an $\epsilon$-close LP instance with multiple optimal solutions.
\end{assumption}
\begin{assumption} \label{ass:resolution}
Like before, let $\mathbf{X}^*(\mathbf{w})$ be the set of optimal LP solutions under cost vector $\mathbf{w}$. Further, let $\mathbf{x}^*(\mathbf{w})\in\mathbf{X}^*(\mathbf{w})$ denote the solution returned by our solver and let $\mathbf{\hat{w}}=\mathbf{w}+\epsilon \ \nabla_{\mathbf{w}} F$ be the perturbed cost vector for some function $F$ and $\epsilon>0$. We assume $\mathbf{x}^*(\mathbf{w})\neq \mathbf{x}^*(\mathbf{\hat{w}})$. In words, the perturbed LP instance returns a different optimal solution.
\end{assumption}
\noindent Arguably, both assumptions are fairly weak and compare to what can be found in standard adversarial learning literature, yet, it is crucial to make them explicit both for transparency on the given setting and for proving our theorem of HCAs on LPs. Now, we are set to give the technical description of what a Hidden Confounder Attack really is:
\begin{defin} \label{def:hca}
Let $\phi_{\mathcal{M}}$ be an LP-parameterization based on SCM $\mathcal{M}$. We call $\phi_{\mathcal{M}}$ prone to Hidden Confounder Attacks if there exists an injective function $h: \mathbf{X}^*\rightarrow \mathbb{R}$ with properties
\begin{enumerate}
\setlength{\itemindent}{.7in}
    \item $h(\mathbf{x}^*)=f(\displaystyle\bigoplus_{i\in\mathbf{x}^*} \mathcal{M}^{\prime}_C(i))$ and
    \item $\exists \mathbf{\hat{w}}\ldotp(h(\mathbf{x}^*(\mathbf{w}))\neq h(\mathbf{x}^*(\mathbf{\hat{w}})))$
\end{enumerate}
for some function $f: \mathbb{R}\rightarrow \mathbb{R}$, aggregation function $\bigoplus$ over units $i$ active in $\mathbf{x}^*$ (e.g.\ the sum for all matched ``workers"), and LP cost vector $\mathbf{w}$ where 
\begin{equation*}
    \mathcal{M}^{\prime}_C: \mathbb{N}\rightarrow \text{Val}(C)
\end{equation*} is the projection of a unit $i$ to its respective confounder value $\mathcal{M}^{\prime}_C(i) \in \text{Val}(C)$ where $\mathcal{M}^{\prime}$ is an alternate SCM containing $C$. That is, $C$ is a hidden confounder of $\mathcal{M}$.
\end{defin}
\noindent In simple terms, property 1 in Def.\ref{def:hca} refers to the uncountable number of functions that can leverage information on the hidden confounder by using the alternate SCM $\mathcal{M}^{\prime}$ to distinguish between different optimal LP solutions which is required by property 2 (since otherwise, there would be no observed difference, alas no attack). Finally, we can provide our key result, that we have encountered previously in informal terms, in its complete formal version.
\setcounter{theorem}{0}
\begin{theorem}[\textbf{HCAs on LPs}{\normalfont, formal}] \label{thm:confhca}
Let $\phi_{\mathcal{M}}$ be an \textbf{integral} LP-parameterization based on SCM $\mathcal{M}$, then we have:
\begin{equation*}
    \mathcal{M} \ \textnormal{is causally insufficient} \iff \phi_{\mathcal{M}} \ \textnormal{is prone to HCA}.
\end{equation*}
\end{theorem}
\begin{proof}
``${\implies}$": As discussed in Sec.\ref{sec:bg}, for an SCM $\mathcal{M}=\inner{\mathbf{U},\mathbf{V},\mathcal{F},P(\mathbf{U})}$ to be causally insufficient there must exist at least one hidden confounder, denoted $C$, that is not an endogenous variable, $C\not\in\mathbf{V}$. Therefore, for any LP-parameterization $\phi_{\mathcal{M}}$ and any LP cost vector $\mathbf{w}$, the latter also doesn't depend on $C$. Then, take an alternate $\mathcal{M}^{\prime}=\inner{\mathbf{U}^{\prime},\mathbf{V}^{\prime},\mathcal{F}^{\prime},P^{\prime}(\mathbf{U})}$ for which $C\in\mathbf{V}^{\prime}$ and construct $h$ as described by property 1 from Def.\ref{def:hca}, which is guaranteed to exist since $\phi_{\mathcal{M}}$ is integral. Through Assumption \ref{ass:epsball} we have guaranteed multiple optimal LP solutions for $\textnormal{LP}(\mathbf{x}; \mathbf{w},\mathbf{A},\mathbf{b})$ to choose from. On the other hand, through Assumption \ref{ass:resolution} we can perturb said LP cost vector $\mathbf{w}$ such that $\mathbf{x}^*(\mathbf{w})\neq \mathbf{x}^*(\mathbf{\hat{w}})$. Since $h$ is injective, we have that $h(\mathbf{x}^*(\mathbf{w}))\neq h(\mathbf{x}^*(\mathbf{\hat{w}}))$ which is property 2 of Def.\ref{def:hca} completing the implication.\\
``${\impliedby}$": Trivial, since HCA (Def.\ref{def:hca}) are defined as attacks that exploit hidden confounders.
\end{proof}
\noindent This fundamental Theorem of our formalism on HCA guarantees us that the existence of hidden confounders implies the susceptibility of LPs to HCAs. In fact, we can even \emph{construct an uncountable number} of HCAs based on said confounders. We further argue that the HCAs that follow from Thm.\ref{thm:confhca} are highly non-trivial in the sense that they exploit information ``outside of the data'' i.e., assuming that the human modeller only uses a causally insufficient $\phi_{\mathcal{M}}$, then an adversary is guaranteed a chance to exploit his/her better knowledge on the true, underlying SCM to perform an attack. Also, a simple corollary of Thm.\ref{thm:confhca} is that LA and SP are \emph{always} prone to HCAs since they have integral $\phi_{\mathcal{M}}$ and \Ra{the odds are stacked against the model under consideration $\mathcal{M}$ actually being the true, underlying SCM $\mathcal{M}^*$ (that is, in most practical cases we will encounter the situation where $\mathcal{M}\neq\mathcal{M}^*$)}. Another way of looking at HCA is possibly even more intriguing: since we need to take care of modelling assumptions to prevent HCA, the modelling of LPs ultimately becomes a \emph{causal problem} since causality is mainly concerned with the discussion of modelling assumptions (usually for identifiability of causal quantities, whereas in this case for the robustness of an LP towards HCAs).

\section{Empirical Evidence}
This section is purely dedicated to discussing the existence of relevant HCA apart from the example that we have already discussed with the LA problem of assigning people to free vaccination spots. We first cover a SP for travelling between cities and then also a general LP for a real world inspired model of an energy system.

\begin{figure*}[!t]
    \centering
    \includegraphics[width=\textwidth]{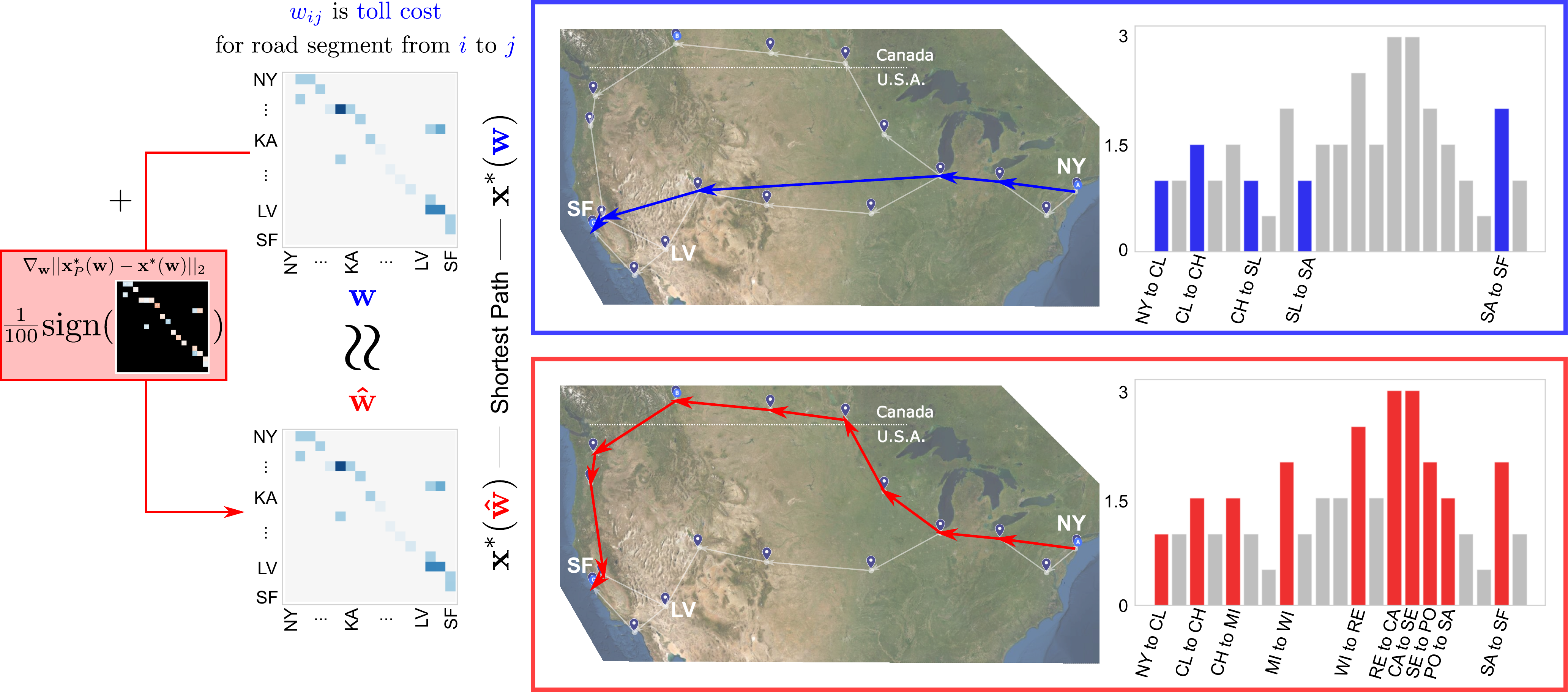}
    \caption{\Ra{\textbf{Another Example: Increased $CO_2$ Emissions.} 
    The edges in the graph represent tolls to be paid for travelling a given road segment. The hidden confounder are $CO_2$ emissions. We visualize the results of performing a HCA using said hidden confounder on our SP-LP instance. Our HCA reveals that travelling via Canada instead of mid-US will amount to the same total travel toll to be paid but the $CO_2$ emissions drastically deviate between the solutions. (Best viewed in Color.)
    }}
    \label{fig:NY-to-SF}
\end{figure*}

\subsection{Shortest Path LP}
The caption of this example might be the following captivating headline for a newspaper article: ``Travelling from New York City to San Francisco...\emph{via Canada}?''. Surely, when travelling between NYC and SF, one would not take the detour over Canada, but this is exactly what occurs in our example in which we perform a HCA on a corresponding LP. Like before with our HCA on the LA-LP with the vaccine example, we depict the application of the HCA visually within a figure. This new example is being showcased in Fig.\ref{fig:NY-to-SF} which is concerned with a classical Shortest Path (SP) problem and we discuss the details in the following. In a corresponding real world setting for the example, we might consider the development of an autonomous car (since arguably any conscious driver would surely notice passing border patrol when actually travelling such a route) to argue for realism. We let the developmental autonomous car travel within North America from New York City (NY) to San Francisco (SF). Our SP has the intention of reducing overall toll costs for the optimal route. Therefore, the SP-LP is not concerned with the shortest route in terms of actual travelling distance but in terms of actual toll costs accumulated (noting this distinction is important). From experience, it is known that these toll costs can be hefty. Our LP cost $w_{ij} \in \mathbb{R}_{>0}$ represents the toll cost when travelling on any road segment from $i$ to $j$. In this example, we know the toll costs for a relevant set of road segments within North America where the Canadian road toll policy is comparably modest when compared to the one in the US. Our LP model subsequently solves any given SP problem instance, fully parameterized by the directed acyclic graph (DAG) $\mathbf{w}\in\mathbb{R}^{n\times n}$ with $n$ being the total number of different cities we have specified, returning $\mathbf{x}_{US}:=\mathbf{x}^*(\mathbf{w})\in[0,1]^{n\times n}$ suggesting a route through the mid-US. We change this result by constructing a HCA. The HCA leads to a minimal perturbation of the original DAG (representing the toll costs), that is $\mathbf{\hat{w}}\approx \mathbf{w}$, but our solver now chooses an alternate solution $\mathbf{x}_{CA}:=\mathbf{x}^*(\mathbf{\hat{w}})$ suggesting a route across the border via Canada.\footnote{Remember, it is essential for the technical Assumption \ref{ass:epsball} and \ref{ass:resolution}, discussed in the previous section, to hold.} While evidently the alternate route deviates strongly in terms of selected road segments, mathematically $SHD(\mathbf{x}_{US}, \mathbf{x}_{CA}) \gg 0$ where $SHD(\cdot, \cdot)\in\mathbb{N}$ is the Structural Hamming Distance, our model is in fact trustfully returning the optimal solution. In other words, cost-wise the statement $\mathbf{w\tran} \mathbf{x}_{US}\approx \mathbf{\hat{w}\tran} \mathbf{x}_{CA}$ holds. Nonetheless, the aforementioned deviation in terms of the resulting binary codes lends itself to a severe consequence in terms of the hidden confounder i.e., with respect to $CO_2$ emissions. Like in the main text, we construct an HCA with function $h$ accordingly. The hidden confounder is being exploited by the adversary, the alternate optimal solution performs significantly worse: $h(\mathbf{x}_{CA}) \gg h(\mathbf{x}_{US})$. In words, both solutions require approximately the same toll costs and are therefore deemed equivalent in that regard but in terms of $CO_2$ emmissions, the (distance-wise) longer route via Canada is far worse for the environment. By this, we have again provided existential proof that a hidden confounder can more generally define adversarial attacks for mathematical programs beyond the original formulation in the classical setting for classification (and deep networks), making the attack a consequence of not the specific methodology being applied to the problem but problem specification itself.

\begin{table*}[t!]
\begin{align*}
    \min_{Cap, p} \quad &c_{PV} \times Cap_{PV} + c_{Bat}\times Cap^S_{Bat} + \sum_t c_{Ele}\times p_{Ele}(t) + \sum_t c_{Gas}\times p_{Gas}(t) \\
    s.t. \quad & p_{Ele}(t) + p_{PV}(t) + p^{out}_{Bat}(t) - p^{in}_{Bat}(t) + p_{Gas}(t) = D(t), \forall t \quad 0 \leq p_{Ele} \\
    & p^S_{Bat}(t) = p^S_{Bat}(t-1) + p^{out}_{Bat}(t) - p^{in}_{Bat}(t), t\in 2,\dots,T\\
    & 0 \leq p_{PV}(t) \leq Cap_{PV}\times avail_{PV}(t) \times \delta t, \forall t \\
    & 0 \leq p^{in}_{Bat}(t), p^{out}_{Bat}(t) \leq Cap_{Bat}, \forall t\\
    & 0 \leq p_{Gas}(t) \leq U_{Gas}, \forall t\\
    & p^{S}_{Bat}(0) = 0
\end{align*} 
\caption{\textbf{Real-world Optimization Modelling Example: 1-year Energy Systems LP for an Average Household.} A large LP that unrolls for 8760 time steps (8760 hours = 1 year). Model based on \citep{schaber2012transmission}, the quantities represent: Cost for Photovoltaics $c_{PV}$ (€/kW), Battery $c_{Bat}$ (€/kWh), Market Electricity $c_{Ele}$ (€/kWh), Gas $c_{Gas}$ (€/kWh), and the total Demand $D$ (kWh/Year).}
\label{tab:energylpformula}
\end{table*}\vspace{.25cm}
\begin{table}[t]
    \centering
    \scalebox{.95}{
\begin{tabular}{*{8}{|c}||c|}\hline
Dem.~($h$) & $Cap_{PV}$ & $Cap_{Bat}$ & Self-Gen.\ & TOTEX & CAPEX & $Con_{Gas}$ & $Con_{Ele}$ & $w_{PV}$\\
	\hline
	\hline
	3000 & 1.76 & 2.45 & 0.42 & 597.41 & 161.64 & 1.70 & 1743.06 & .005\\
	\hline
	3000 & \textcolor{green}{7.15} & 4.78 & 0.66 & 468.24 & 214.87 & 1.95 & \textcolor{red}{1013.49} & .001\\
	\hline
\end{tabular}}
\vspace{.25cm}
  	\caption{\textbf{Dominating Technologies}. We perform an approximate HCA to reveal a new solution that highlights the fact that PV end up as a new ``dominating technology.'' Price perturbations in ($w_{PV}$) have boosted PV production $Cap_{PV}$ (green), which then again can be argued leads to a significant increase in risk of fire or injuries for the workers installing the panels.
    \label{table:price}}
    \vspace{-.5cm}
\end{table}

\subsection{Energy-Systems LP} \label{subsec:energy}
In this final empirical simulation, we consider a \emph{large scale} LP. Furthermore, this LP is a general LP, so neither LA nor SP and thereby it does not satisfy \emph{integrality}. Our real world based example considers an energy model for modelling the energy portfolio of a single-family house based on actual real world data for demand and commonly used equations from energy systems research \citep{schaber2012transmission} to model the evolution of respective quantities. The examined model considers concepts such as photovoltaics (PV), market electricty and heating gas over a year time frame (in hours) and resembles a simplified version of the TIMES model \cite{loulou2005documentation}. The optimal solution balances the usage of the different technologies for matching the required demand of the family household such that overall cost is being minimized. The specific LP template is given in Tab.\ref{tab:energylpformula}. Note the summation and functional dependencies on $t \in \{0,\dots, 8760\}$ with 1 year = 8760 hours rendering the template a \emph{very large single LP modelling each hour of the year} with well over 40,000 constraints and an objective function with over 17,000 terms. Still, there are some technologies like for instance PV, in their capacity ($Cap_{PV}$), that do not depend on $t$ which would correspond to the real world intuition that one does not decide and subsequently build new PV for any given hour as it poses a single investment fixed in time. Since in this example we cannot specify the underlying SCM to then also identify confounders, as previously done for the LA- and SP-LPs, we need to treat this LP instance differently. We use heuristics for the SCM-part to produce an approximation to HCA to overcome the fact that we are not provided with a reasonable SCM a-priori. We then perform said approximation to produce an HCA that creates the results presented in Tab.\ref{table:price}. We observe the effect which energy-systems researchers call ``dominating technologies,'' where PV is being preferred over market-bought electricity. While we do not have a function $h$ this time to evaluate the difference in solutions for the adversary, we can still make the argument that building this many PV modules comes at an increased risk of fire (which could be considered the hidden confounder in this case). Another possible interpretation would be risk of working injury for the panel installing workers, since installing the panels usually happens at the upper level height of the house. To conclude this example with an important discussion, we want to mention that the limitations on PV-production and Market-buy of electricity act as discrepancy counter-measures that require the system to balance out different technologies i.e., while there will still be dominating technologies under price advantages the maximum skew of the portfolio is naturally being protected from being too drastic as both PV and bought electricity are limited in their `availability' (e.g.~solar exposure, roof capacity, law regulations etc.) and thus cannot be naïvely maximized. In other words, we observe that this LP behaves qualitatively a little different when compared to the LA- and SP-LP examples in the sense that this energy system LP is more `balanced' and thus somehow less prone to HCA. The aforementioned lack of integrality might be one of the reasons, but there is reason to believe based on the previous argument of the \emph{dynamics of the competing technologies} that this might be the main cause. A precise formalization of these aspects is a remaining open problem.

\subsection{\Rb{Synthetic Simulations, Scalability and Key Assumptions}}
\Rb{We've conducted several experiments of synthetic nature on LP problems. These LP problems included well-known integer problems such as matching or shortest path, and further general LP formulations as in the case of energy-systems. While the LP problems themselves are considerably data-independent, the presented HCA formalism relies on SCM as data-generating processes, therefore, sample or data set sizes in the LP parameterizations $\phi_{\mathcal{M}}$ can vary. Unsurprisingly, since Def.\ref{def:phi} puts a constraint on the data set under consideration there will be no downstream influence by the data set size onto the HCA. That is, since hidden confounders are a property of an SCM (on the model level) they do not have an influence on the data set we consider (on the sample level). Similarly, our simulations corroborate on the invariance (or rather orthogonality) of HCA to variable scales as this is a general property of the LPs under consideration e.g., the scale of a variable in the SP setting that denotes ``road segment will be taken, yes/no'' can not be changed. In an analogue argument for the general LP, the variables having well-specified scales is a defining (thus necessary) property e.g.\ the quantities in the energy-system example need to satisfy given physical conservation laws. }

\Rb{Regarding the scalability of HCA, it can be added that the knowledge on a confounder and subsequent $h$ is always $\mathcal{O}(1)$, that is, independent of the size of any given SCM $\mathcal{M}$, if that $\mathcal{M}$ is causally insufficient, then a single hidden confounder is sufficient for constructing a $h$ that acts as HCA. Other than the causal aspect, regular scaling properties of DPOs apply to the construction of HCA. Concerning the usage of different LP solvers for obtaining optimal solutions $\mathbf{x}^*$: the presented approach crucially builds upon DPOs that are characterized by their differentiability. Differentiability being a crucial property for adversarial examples, makes it a necessary condition for HCA as well, as they can be viewed as an extension to LP adversarial-style examples. In conclusion, only differentiable LP solvers can be employed, and to the best of our knowledge no other solvers, apart from the ones employed here, have been studied thus far. Regarding the assumptions, there are more nuanced views to be considered for justification purposes. The existence of an underlying SCM is the foundational assumption in Pearlian causality and can be taken as granted. Similarly, the fact that the underlying SCM is only being approximated by any given model of the data. Therefore, knowledge on hidden confounders is a strong assumption. In our case, this assumption is relaxed since only a \emph{single} hidden confounder need be known. Regarding the two technical assumptions \ref{ass:epsball} and \ref{ass:resolution} in this work, the latter is concerned with a `tie-break' resolution, that is, if there exist multiple solution to a given LP parameterization, then a certain permutation will always favour a certain solution. This is a straightforward assumption since the parameterization can be altered arbitrarily by an $\epsilon$ change. However, the prior assumption which is concerned with the \emph{existence} of multiple solutions within an $\epsilon$-ball is more demanding than in the standard adversarial case since LPs are polytopes of potentially complex shape. Nonetheless, on a conceptual level, the presented showcases for LA/SP demonstrate that various situations of multiple-solution sets for LPs do occur in practice.}


\section{Conclusive Discussion}
Ultimately, we believe HCAs to be a fundamental problem of mathematical optimization---to the same extent that hidden confounding is a fundamental problem of the Pearlian notion of causality (or science in general). It is intriguing that the formal tools from causality allowed for bridging the gap between classical adversarial attacks from deep learning and the first basic class of mathematical optimization namely LPs. While it is arguably of great scientific value to purely study the existence and properties of HCA, naturally, the question arises on the severity of HCA for the real world. We believe that our examples have shown potentially worrisome real-world implications. Especially our lead example in Fig.\ref{fig:leadex} captures the \emph{Zeitgeist} of the pandemic times with the rise of Covid (that hopeful has found an end finally). To thus ask the inverse question on how to defend against HCA is equally important, yet, we believe this question to be ill-posed to begin with. \emph{Essentially, Thm.\ref{thm:confhca} suggests an equivalence on the existence of hidden confounders and such attacks.} Put differently, as long as our model assumptions are imperfect, we are exploitable---again, giving us reason to believe HCA to be of fundamental nature. However, that does not mean that we are doomed to accept that HCA will always be something that can happen anywhere but rather take the opportunity to further study HCA beyond LPs (as initially conjectured, see Conj.\ref{con:mps}) but also alternate notions of attacks, in order to really understand what our assumptions cover and what not. With HCA we have presented one new way of thinking of adversarial style attacks and its serves as a representative example of what we mean by studying these phenomena. While our perspective put causality to good, there might exist other notions of attacks similar to HCA that might in fact not be based on causal knowledge (confounders) to begin with. Since we were able to develop HCA from first principles, by starting from classical adversarial attacks and naïve mappings to LPs, we have good reason to believe in the actual existence of related families of attacks. From a theoretical standpoint the question of whether the integral property (Def.\ref{def:integral}) applied to Thm.\ref{thm:confhca} could be dropped might be interesting for broadening the applicability of HCA, as we saw with our energy systems example where we still achieved some sort of reasonable HCA approximation although integrality did not hold.

On another note, we observe this work to form a triangular relationship to the works of \citep{ilse2021selecting} and \citep{eghbal2020data}. To elaborate: the first publication is concerned, again, with Pearlian causality and sees how it relates to data augmentation, while the second paper bridges augmentation and adversarials---our work can be seen as the missing link that then loops back adversarials to causality. From that perspective, we can clearly see that there seems to be an overarching research theme yet to be uncovered. Separating our discussion from HCA for the moment, we nonetheless are tempted to believe that HCA (although being the focus and motivation behind this work) are \emph{not} the most important discovery in this paper. The concept of LP-parameterization based on SCM (Def.\ref{def:phi}) is an intriguing and original concept that for the first time connects the two seemingly independent notions of causality and mathematical optimization in a non-trivial manner. So, it might turn out to be more fruitful long-term to actually study the properties of mathematical programs which stand in direct relation to the data-generating capabilities of SCM since that might lead to interesting concepts such as the estimation of \emph{causal effects} in, say, constraints of mathematical programs. To put it bluntly, our current thinking imagines future research around LP-SCM relations as even more fundamental than HCA, but ideally both HCA and related attacks should be further studied alongside LP-SCM relations since both have real world implications. Arguably, for causality the latter is more important, whereas for ML the former is more important.

\paragraph{Takeaway and Societal Implications}
Our work seems to suggest (1) that we can have a similar, adversarial phenomenon outside classification and (2) that the current state of ML can be viewed `causal' to the extent that the assumptions are `causal' i.e., $\phi$ summarizes essentially these causal assumptions. Of concern are mostly (a) raising awareness on the issue that adversaries could in fact use HCA as shown in the examples to produce harm and (b) provide incentives for studying the integration of ML with causality since modelling assumptions seem to lie at the core of it. 

\section*{Declarations}
\noindent\textbf{Funding.} This work was supported by the ICT-48 Network of AI Research Excellence Center ``TAILOR'' (EU Horizon 2020, GA No 952215), the Nexplore Collaboration Lab ``AI in Construction'' (AICO) and by the Federal Ministry of Education and Research (BMBF; project ``PlexPlain'', FKZ 01IS19081). It benefited from the Hessian research priority programme LOEWE within the project WhiteBox and the HMWK cluster project ``The Third Wave of AI'' (3AI). The authors thank Jonas Hülsmann and Florian Steinke for providing the LP model for the energy system example. We further thank the anonymous reviewers for their valuable feedback and constructive criticism.

\noindent\textbf{Conflicts of interest.} The authors have no competing interests to declare.

\noindent\textbf{Ethics approval.} Ethics approval was not required for this research.

\noindent\textbf{Consent to participate. / Consent for publication.} Not applicable.

\noindent\textbf{Availability of data and material / Code availability.} We make our code publicly available at \url{https://github.com/zecevic-matej/Hidden-Confounder-Attacks}

\noindent\textbf{Authors' contributions.} Matej Ze{\v{c}}evi{\'c} (MZ), Devendra Singh Dhami (DSD), Kristian Kersting (KK) contributed to conception of the idea. MZ performed the data preparation and experimental analysis. MZ wrote the first draft of the manuscript. DSD and KK wrote sections of the manuscript. All authors contributed to manuscript revision, read, and approved the submitted version.




\begin{appendices}
\vspace{.25cm}
\noindent\textit{\textbf{Reproduction Details.}} \textbf{LA-LP:} For the LA example, with the vaccination matching bias towards the wealthy, we use $N=15$ sampling iterations with temperature parameter $\sigma=0.5$ for the perturbation in the DPO as defined by \cite{berthet2020learning} and an attack step $\epsilon=0.01$ for the final HCA. \textbf{SP-LP:} For the SP example, travelling from NY to SF via Canada, we use more sampling iterations ($N=20$) using a lower temperature ($\sigma=0.25$). \textbf{Real World LP:} The energy systems LP has following parameter specifications:
\begin{wraptable}{r}{5.5cm}
	\centering
	\small
	\begin{tabular}{|c|c|c|c|c|}
		\hline
		$c_{PV}$ & $c_{Bat}$ & $c_{Ele}$ & $D$ & $c_{Gas}$ \\
		\hline
		\hline
		0.005 & 300 & 0.25 & 3000 & 0.25\\
		\hline
		0.001 & 300 & 0.25 & 3000 & 0.25\\
		\hline
	\end{tabular}
	\caption{\textbf{Parameterization Energy-System}. Cost for Photovoltaics $c_{PV}$ (€/kW), Battery $c_{Bat}$ (€/kWh), Market Electricity $c_{Ele}$ (€/kWh), Gas $c_{Gas}$ (€/kWh), and the total Demand $D$ (kWh/Year).
	}\label{table:details-energy}
\end{wraptable}

\noindent All experiments are being performed on a MacBook Pro (13-inch, 2020, Four Thunderbolt 3 ports) laptop running a 2,3 GHz Quad-Core Intel Core i7 CPU with a 16 GB 3733 MHz LPDDR4X RAM on time scales ranging from a few minutes (e.g.\ evaluating LA/SP examples) up to a few hours (e.g.\ energy systems real world example). Code link in Sec.\ref{sec1}.

\end{appendices}

\bibliography{references.bib}
\end{document}